\documentclass[11pt]{article}

\usepackage{amssymb}
\usepackage{amsmath}
\usepackage{amsthm}
\usepackage{accents}
\usepackage{enumerate}
\usepackage{algorithm}
\usepackage{graphicx}
\usepackage{url}
\usepackage{bm}
\usepackage{color}

\theoremstyle{plain}

\newtheorem{theo}{Theorem}
\newtheorem{prop}{Proposition}
\newtheorem{lemm}{Lemma}
\newtheorem{coro}{Corollary}

\newtheorem{assump}{Assumption}

\theoremstyle{definition}

\newtheorem{remark}{Remark}

\def\0{\bm{0}}
\def\1{\bm{1}}
\def\2{\bm{2}}
\def\3{\bm{3}}
\def\4{\bm{4}}
\def\5{\bm{5}}
\def\6{\bm{6}}
\def\7{\bm{7}}
\def\8{\bm{8}}
\def\9{\bm{9}}

\def\a{\bm{a}}

\def\e{\bm{e}}
\def\f{\bm{f}}
\def\g{\bm{g}}

\def\k{\bm{k}}

\def\n{\bm{n}}

\def\p{\bm{p}}

\def\s{\bm{s}}
\def\t{\bm{t}}

\def\v{\bm{v}}

\def\x{\bm{x}}

\def\A{\bm{A}}
\def\B{\bm{B}}
\def\C{\bm{C}}
\def\D{\bm{D}}

\def\F{\bm{F}}
\def\G{\bm{G}}

\def\I{\bm{I}}

\def\K{\bm{K}}
\def\L{\bm{L}}

\def\N{\bm{N}}

\def\P{\bm{P}}
\def\Q{\bm{Q}}

\def\S{\bm{S}}

\def\U{\bm{U}}
\def\V{\bm{V}}
\def\W{\bm{W}}

\def\IC{\mathcal{I}}

\def\SC{\mathcal{S}}

\def\Pib{\mbox{\bm{$\Pi$}}}
\def\Sigmab{\mbox{\bm{$\Sigma$}}}
\def\Lambdab{\mbox{\bm{$\Lambda$}}}

\def\Real{\mbox{$\mathbb{R}$}}

\def\mmin{\mbox{\scriptsize min}}
\def\mmax{\mbox{\scriptsize max}}

\makeatletter
\def\widebar{\accentset{{\cc@style\underline{\mskip10mu}}}}
\def\Widebar{\accentset{{\cc@style\underline{\mskip8mu}}}}
\makeatother

\def\wb{\widebar}
\def\wh{\widehat}
\def\wt{\widetilde}
\def\equivSym{\Leftrightarrow}

\newcommand{\by}[2]{$#1$-by-$#2$}

\def\MVEE{\mbox{$\mathsf{P}$}}
\def\transMVEE{\mbox{$\mathsf{Q}$}}

\title{
Robustness Analysis of Preconditioned Successive Projection Algorithm
for General Form of Separable NMF Problem
}

\author{Tomohiko Mizutani 
\thanks{
Department of Industrial Engineering and Management, 
Tokyo Institute of Technology,
2-12-1-W9-69, Ookayama, Meguro, Tokyo, 152-8552, Japan. 
{\tt mizutani.t.ab@m.titech.ac.jp}}}

\date{\today}

\begin{document}

\maketitle

\begin{abstract}
The successive projection algorithm (SPA) has been known to work well 
for separable nonnegative matrix factorization (NMF)
problems arising in applications, 
such as topic extraction from documents and endmember detection in hyperspectral images.
One of the reasons is in that the algorithm is robust to noise.
Gillis and Vavasis showed in 
[SIAM J. Optim., 25(1), pp.~677--698, 2015]
that a preconditioner can further 
enhance its noise robustness.
The proof rested on the condition that the dimension $d$ and factorization rank $r$
in the separable NMF problem coincide with each other.
However, it may be unrealistic to expect that the condition holds 
in separable NMF problems appearing in actual applications;
in such problems, $d$ is usually greater than $r$.
This paper shows, without the condition $d=r$, that the preconditioned SPA is robust to noise.


\end{abstract}

\section{Introduction}
\label{Sec: intro}

A \by{d}{m} nonnegative matrix $\A$
is said to be {\it separable} if 
it has a decomposition of the form
\begin{equation} \label{Eq: separable matrix}
\A = \F\W \ \mbox{for} \ \F \in \Real^{d \times r}_+ \ 
 \mbox{and} \ \W = (\I, \K)\Pib \in \Real^{r \times m}_+
\end{equation}
where $\I$ is an \by{r}{r} identity matrix, $\K$ is an \by{r}{(m-r)} nonnegative matrix, and 
$\Pib$ is an \by{m}{m} permutation matrix.
Here, we call $\F$ the {\it basis matrix} of $\A$
and $r$  the {\it factorization rank}.
The separable nonnegative matrix factorization problem is stated as follows.
\begin{quote}
 {\bf (Separable NMF Problem}) \
 Let $\A$ be of the form given in (\ref{Eq: separable matrix}).
 Find an index set $\IC$ with $r$ elements such that $\A(\IC)$ coincides with the basis matrix $\F$.
\end{quote}
The notation $\A(\IC)$ denotes the submatrix of $\A$ whose column indices are in $\IC$; 
in other words, $\A(\IC) = (\a_i : i \in \IC)$ for the $i$th column vector $\a_i$ of $\A$.
We use the abbreviation NMF to refer to nonnegative matrix factorization.
The problem above can be thought of as a special case of  NMF problem.
The NMF problem is intractable, and in fact, was shown to be NP-hard in \cite{Vav09}.
The authors of \cite{Aro12a} proposed to put an assumption, called separability.
The separability assumption turns it into a tractable problem referred to as a separable NMF problem.
Although the assumption may restrict the range of applications,
it is known that separable NMF problems nonetheless can be used 
for the purpose of topic extraction from documents \cite{Aro12b, Aro13, Miz14}
and endmember detection in hyperspectral images \cite{Gil14, Gil15}.

Several algorithms have been developed for solving the separable NMF problem.
One of our concerns is how robust these algorithms are to noise,
since it is reasonable to suppose that the separable matrix contains noise 
in separable NMF problems arising from the applications mentioned above.
We consider an algorithm for solving a separable NMF problem
and suppose that the separable matrix contains noise.
If the algorithm can identify a matrix close to the basis matrix,
we say that it is {\it robust to noise}.

The successive projection algorithm (SPA) 
was originally proposed in \cite{Ara01} in the context of chemometrics.
Currently, the algorithm and its variants are used for 
topic modeling, document clustering and hyperspectral image unmixing.
Gillis and Vavasis showed in \cite{Gil14} that SPA is robust to noise
and presented empirical results  suggesting that the algorithm 
is a promising approach to hyperspectral image unmixing.
The theoretical results implied that
further improvement in noise robustness can be expected
if we can make the condition number of the basis matrix smaller.
Hence, they proposed in \cite{Gil15} to use a preconditioning matrix
for reducing the condition number of the basis matrix.
They showed that the noise robustness of SPA is improved 
by using a preconditioner.
The proof rested on the condition that the dimension $d$ and factorization rank $r$ 
in a separable matrix coincide with each other.
However, it may be unrealistic to expect that the condition holds 
in separable NMF problems derived from actual applications.
In such a situation,  $d$ is usually greater than $r$. 
For instance, we shall consider extraction of topics from a collection 
of newspaper articles.
This task can be modeled as a separable NMF problem.
In the problem, the dimension $d$ and factorization rank $r$ of a separable matrix 
correspond to the number of articles and topics, respectively. 
It would be rare that $d$ is close to $r$ but usual that $d$ is greater than $r$.

The aim of this paper is to show, without the condition $d = r$, 
that the preconditioned SPA is robust to noise.
The statement of our result is  in Theorem \ref{Theo: main}. 
It can be used as a guide for seeing how robust 
the algorithm is to noise 
when handling separable NMF problems derived from actual applications.

The rest of this paper is organized as follows.
In Section \ref{Sec: noise robustness of precond SPA},
we review SPA and the preconditioned one
and  describe the results of the noise robustness 
obtained in a series of studies by Gillis and Vavasis.
After that, we describe our analysis of the preconditioned SPA,
comparing our results with those of Gillis and Vavasis. 
Our analysis is shown in Section \ref{Sec: analysis}.

\subsection{Notation and Terminology} \label{Subsec: Notation}
A real matrix is said to be  {\it nonnegative} if all of its elements are nonnegative.
Here, we use the symbol $\Real^{d \times m}$ to represent the set of \by{d}{m} real matrices, 
and $\Real^{d \times m}_+$ the set of \by{d}{m} nonnegative matrices.
The identity matrix is denoted by $\IC$ and the permutation matrix by $\Pib$.
The vector of all ones is denoted by $\e$ and the $i$th unit vector by $\e_i$.
We shall use the capital upper-case letter $\A$ to denote a matrix.
The lower-case letter with subscript $\a_i$ indicates the $i$th column.
We denote the transpose by $\A^\top$, the rank by $\mbox{rank}(\A)$ and 
the matrix norm by $\|\A\|$.
In particular, the matrix $2$-norm and the Frobenius norm are written as $\|\A\|_2$ and $\|\A\|_F$.
We use the symbol $(\A; \B)$ for matrices $\A \in \Real^{d \times m}$ and $\B \in \Real^{d' \times m}$
to represent the matrix 
\begin{equation*}
 \left(
 \begin{array}{c}
  \A \\
  \B
 \end{array}
 \right) \in \Real^{(d + d') \times m},
\end{equation*}
and the symbol $\mbox{diag}(a_1, \ldots, a_m)$ for numbers $a_1, \ldots, a_m$ to represent
the diagonal matrix 
\begin{equation*}
 \left(
 \begin{array}{cccc}
  a_1 &        &  \\
      & \ddots &  \\
      &        & a_m
 \end{array}
 \right) \in \Real^{m \times m}.
\end{equation*}
For two real numbers $a$ and $b$, the symbol $\min(a,b)$ indicates the smaller value.

\subsection{Tools from Linear Algebra}
Any real and complex matrix has a singular value decomposition (SVD).
We will use the SVD of a real matrix in our subsequent discussion.
Let $\A \in \Real^{d \times m}$.
The SVD of $\A$ can be written as
\begin{equation} \label{Eq: SVD}
 \A = \U\Sigmab \V^\top.
\end{equation}
$\U$ and $\V$ are \by{d}{d} and \by{m}{m} orthogonal matrices.
In particular, the column vectors of $\U$ and $\V$ are called 
the {\it left singular vectors} and {\it right singular vectors} of $\A$.
$\Sigmab$ is a \by{d}{m} diagonal matrix.
If $d \le m$, it has the form $(\mbox{diag}(\sigma_1, \ldots, \sigma_d), \0)$ 
for a \by{d}{(m-d)} zero matrix $\0$;
otherwise, $(\mbox{diag}(\sigma_1, \ldots, \sigma_m); \0)$ 
for a \by{(d-m)}{m} zero matrix $\0$.
Let $t = \min(d,m)$.
It is known that the diagonal elements $\sigma_1, \ldots, \sigma_t$
are all nonnegative.  
These elements are called {\it singular values} of $\A$.
By changing the order of columns in $\U$ and $\V$,
we can arrange the singular values in descending order.
Therefore, throughout this paper, we always assume that $\sigma_1 \ge \cdots \ge \sigma_t$ in $\Sigmab$.
We use the symbols $\sigma_{\mmin}(\A)$ and $\sigma_{\mmax}(\A)$ to denote
the smallest and largest singular values among them; 
in other words, $\sigma_{\mmin}(\A) = \sigma_t$ and $\sigma_{\mmax}(\A) = \sigma_1$.
We define the {\it condition number} of $\A$
as $\sigma_{\mmax}(\A) / \sigma_{\mmin}(\A)$, and use $\kappa(\A)$ to denote it.

Let $\A$ be a \by{d}{d} symmetric positive definite matrix.
Due to the positive definiteness,
$\A$ has an eigenvalue decomposition such that $\A = \U \Lambdab \U^\top$
where $\U$ is a \by{d}{d} orthogonal matrix,
and $\Lambdab$ is a \by{d}{d} diagonal matrix
with positive  elements $\lambda_1, \ldots, \lambda_d$.
We define the square root of $\A$ as $\U \Lambdab^{1/2} \U^\top$ 
where $\Lambdab^{1/2} = \mbox{diag}(\lambda_1^{1/2}, \ldots, \lambda_d^{1/2})$, 
and use $\A^{1/2}$ to denote it.
We use the symbol $\A \succ \0$ to mean that $\A$ is positive definite.

\section{Noise Robustness of the Preconditioned SPA for Separable NMF Problems}
\label{Sec: noise robustness of precond SPA}

This section consists of three subsections.
We start by examining separable NMF problems from a geometric point of view
and summarize each step of SPA in Algorithm \ref{Alg: SPA}.
The geometric interpretation of the problems will help us to understand the notion behind SPA.
The analysis of the noise robustness of SPA by \cite{Gil14} is described 
in Theorem \ref{Theo: robustness of SPA}.
Next, assuming the condition $d=r$, 
we intuitively explain why the noise robustness of SPA can be enhanced 
by using a preconditioning matrix.
Then we describe the result 
for the preconditioned SPA by \cite{Gil15} in Theorem \ref{Theo: robustness of precond SPA}.
We summarize the algorithm of preconditioned SPA in Algorithm \ref{Alg: precond SPA},
which works without the condition $d=r$.
Finally, we describe our results on the robustness of Algorithm \ref{Alg: precond SPA}
to noise and compare them with the results of 
Theorem \ref{Theo: robustness of precond SPA} of \cite{Gil15}.

\subsection{Review of SPA} \label{Subsec: review of SPA}
Separable NMF problems have a geometric interpretation.
Let $\A$ be of the form given in (\ref{Eq: separable matrix}).
Without loss of generality,
we can assume that any column vector $\k_i$ of $\K$ satisfies $\e^\top \k_i = 1$,
since $\A = \F\W \equivSym \A\D_1  = \F\D_2 \D_2^{-1} \W \D_1$ 
for nonsingular diagonal matrices $\D_1$ and $\D_2$.
In addition, assume that $\mbox{rank}(\F) = r$.
Under these assumptions, the convex hull of the column vectors of $\A$ is 
an $(r-1)$-dimensional simplex in $\Real^d$, 
and the vertex corresponds to each column vector of $\F$.
Accordingly, we can restate the separable NMF problem as follows;
find all vertices of the convex hull of the column vectors of $\A$.
In \cite{Gil14, Gil15, Miz14},
we can find a further explanation of the problem.

SPA is designed on the basis of the geometric interpretation of separable NMF problems.
The first step finds $\a_{i^*}$ among the column vectors $\a_1, \ldots, \a_m$ of $\A$ 
that maximizes the convex function $f(\x) = \|\x\|_2^2$
and projects $\a_1, \ldots, \a_m$ onto the orthogonal space to $\a_{i^*}$.
This procedure is repeated until $r$ column vectors are found.
As pointed out in \cite{Gil14}, 
SPA has a connection to QR factorization with column pivoting by \cite{Bus65}.
Algorithm \ref{Alg: SPA} describes each step of SPA.
We may see why SPA can find $\F$ from $\A$
by recalling the following property;
given the set of points in a polytope, including all the vertices,
the maximum of a strongly convex function over the set is attained at one of the vertices.

\begin{algorithm}
 \caption{SPA}  \label{Alg: SPA}
 \textbf{Input:} 
 A \by{d}{m} real matrix $\A$ and a positive integer $r$. \\
 \textbf{Output:} 
 An index set $\IC$.
 \begin{enumerate}
  \item[\textbf{1:}] 
	       Initialize a matrix $\S$ as $\S \leftarrow \A$, 
	       and an index set $\IC$ as $\IC \leftarrow \emptyset$.

  \item[\textbf{2:}] 
	       Find an index $i^*$ such that $i^* = \arg \max_{i=1, \ldots, m} \|\s_i\|_2^2$
	       for the column vector $\s_i$ of $\S$.

  \item[\textbf{3:}] 
	       Set $\t \leftarrow \s_{i^*}$.
	       Update $\S$ as 
	       \begin{equation*}
		\S \leftarrow \biggl(\I - \frac{\t \t^\top }{\|\t\|_2^2}\biggr) \S,
	       \end{equation*}
	       and $\IC$ as 
	       \begin{equation*}
		\IC \leftarrow \IC \cup \{i^*\}.
	       \end{equation*}

  \item[\textbf{4:}]
	       Go back to step 2 if $|\IC| < r$; otherwise,  
	       output $\IC$, and then terminate.

 \end{enumerate}
\end{algorithm}

Now let us describe the analysis of Algorithm \ref{Alg: SPA} 
given by Gillis and Vavasis in \cite{Gil14}.
We put the following assumption on a \by{d}{m} real matrix $\A$.
\begin{assump} \label{Assump}
 $\A$ can be decomposed into $\A = \F \W$
 for $\F \in \Real^{d \times r}$ and $\W = (\I, \K)\Pib \in \Real^{r \times m}_+$
 where $\I$, $\K$ and $\Pib$ are the same as those of (\ref{Eq: separable matrix}).
 $\F$ and the column vector $\k_i$ of $\K$ satisfy the following conditions.
 \begin{enumerate}[{\normalfont (a)}]
  \item $\mbox{rank}(\F) = r$.
  \item $\e^\top \k_i \le 1$ for all $i = 1, \ldots, m-r$.
 \end{enumerate}
\end{assump}
Assumption \ref{Assump} corresponds to that made in \cite{Gil14, Gil15}.
Note that the decomposition of $\A$ in the assumption is not exactly the same as 
that of (\ref{Eq: separable matrix}).
$\F$ in (\ref{Eq: separable matrix}) is a nonnegative matrix but 
$\F$ in the assumption is not necessarily a nonnegative one.
As we mentioned in the first part of this section,
from the relation 
$\A = \F\W \equivSym \A\D_1  = \F\D_2 \D_2^{-1} \W \D_1$ for nonsingular diagonal matrices $\D_1$ and $\D_2$,
Assumption \ref{Assump}(b) can be assumed without loss of generality.
Under Assumption \ref{Assump}, Gillis and Vavasis showed in \cite{Gil14}
that Algorithm \ref{Alg: SPA} on the input $(\A, r)$ returns $\IC$ such that $\A(\IC) = \F$.
Furthermore, they showed that it is robust to noise.
Suppose that a separable matrix $\A$ of (\ref{Eq: separable matrix}) 
contains a noise matrix $\N \in \Real^{d \times m}$ 
such that 
\begin{equation} \label{Eq: near-separable matrix}
 \wt{\A} = \A + \N.
\end{equation}
We call $\wt{\A}$ a {\it near-separable matrix}, and $\N$ a {\it noise matrix}.
Their analysis tells us that running Algorithm \ref{Alg: SPA} 
on the input $(\wt{\A}, r)$ returns $\IC$ such that 
$\wt{\A}(\IC)$ is close to $\F$ if the size of $\N$ is small.
The formal statement is  as follows.

\begin{theo}[Theorem 3 of \cite{Gil14}] \label{Theo: robustness of SPA}
 Let $\wt{\A} = \A + \N$ for $\A \in \Real^{d \times m}$ and $\N \in \Real^{d \times m}$.
 Suppose that $r \ge 2$ and $\A$ satisfies Assumption \ref{Assump}.
 If $\n_i$ of $\N$ satisfies
 $\|\n_i\|_2 \le \epsilon$ for all $i = 1, \ldots, m$ with
 \begin{equation*}
  \epsilon < \min \Biggl( \frac{1}{2\sqrt{r-1}}, \frac{1}{4} \Biggr)
   \frac{\sigma_{\mmin}(\F)}{1 + 80 \kappa(\F)^2},
 \end{equation*}
 then, Algorithm \ref{Alg: SPA} with the input $(\wt{\A}, r)$ returns
 the output $\IC$ such that there is an order of the elements in $\IC$ satisfying
 \begin{equation*}
  \|\wt{\a}_{\IC(j)} - \f_j\|_2  \le (1 + 80 \kappa(\F)^2 ) \epsilon
 \end{equation*}
 for all $j = 1, \ldots, r$.
\end{theo}
The notation $\IC(j)$ represents the $j$th element of $\IC$
for a set $\IC$ whose elements are arranged in some order.
Throughout this paper, we will use the notation $\IC(j)$ to refer to the $j$th element 
in the ordered elements of $\IC$.
In the theorem, $\wt{\a}_{\IC(j)}$ is the $\IC(j)$th column vector 
of $\wt{\A}$, and $\f_j$ the $j$th column vector of $\F$.
The statement of the above theorem does not completely match 
that of Theorem 3 of \cite{Gil14}. 
Let us remark on that.

\begin{remark}
Theorem 3 of \cite{Gil14} is described by using 
$L$, $\mu$, and $K(\F)$ that are not found in the above theorem.
Let $f$ be a strongly convex function. Then,
$L$ corresponds to the Lipschitz constant of $f$, 
and $\mu$ is the parameter associated with the strong convexity of $f$.
We have $L = \mu$
since we consider the case in which $f(\x) = \|\x\|_2^2$.
$K(\F)$ is defined as $K(\F) = \max_{j=1, \ldots, r} \|\f_j\|_2$.
From the definition, we have $K(\F) \le \sigma_{\mmax}(\F)$.
Therefore, the above theorem follows from Theorem 3 of \cite{Gil14}.
\end{remark}

\subsection{Preconditioned SPA} \label{Subsec: precond SPA}

Consider a near-separable matrix $\wt{\A}$ of (\ref{Eq: near-separable matrix}). 
Theorem \ref{Theo: robustness of SPA} suggests that, 
if one restricts the condition number of the basis matrix $\F$ to be close to one,
we may expect that 
the allowed range size of $\|\n_i\|_2$ increases 
and the difference between $\F$ and $\wt{\A}(\IC)$ decreases.
Assume that $\A$ of $\wt{\A}$ satisfies Assumption \ref{Assump}.
Let $\Q$ be a \by{d}{d} nonsingular matrix.
Then, the multiplication of $\wt{\A}$ by $\Q$ yields $\Q\wt{\A} = \Q \F \W + \Q \N$.
The assumption still remains valid for $\Q\F$ due to the nonsingularity of $\Q$.
Accordingly,
if we can construct $\Q$ so as to decrease the condition number of $\F$,
the noise robustness of SPA may be improved 
by performing SPA on the input $(\Q\wt{\A}, r)$ instead of $(\wt{\A}, r)$.
In \cite{Gil15}, Gillis and Vavasis proposed a procedure 
for constructing such a preconditioning matrix $\Q$.
Here, we should pay attention to the fact that, 
even if $\Q$ decreases the condition number of $\F$, 
the amount of noise could be expanded up to factor $\|\Q\|$ 
since $\|\Q\N\| \le \|\Q\| \|\N\|$.

\subsubsection{Case of $d=r$}
\label{Subsubsec: d=r}

Now let us explain the procedure for constructing a preconditioning matrix in \cite{Gil15}.
For simplicity, we will consider the noiseless case on $\wt{\A}$. 
That is, $\wt{\A} = \A$.
We assume that $\A$ satisfies Assumption \ref{Assump}, and in addition, 
assume that the dimension $d$ and the factorization rank $r$ coincide with each other.
Under these assumptions, 
$\A$ is an \by{r}{m} separable matrix and has an \by{r}{r} basis matrix $\F$.
We set $\SC = \{\a_1, \ldots, \a_m\}$ for the column vectors $\a_1, \ldots, \a_m$ of $\A$,
and consider the optimization problem,
\begin{equation*} 
 \begin{array}{lll}
  \MVEE(\SC): 
   & \mbox{minimize}    
   & -\log \det(\L), \\
   & \mbox{subject to}  
   & \a^\top \L \a  \le 1
   \ \mbox{for all} \ \a \in \SC, \\
   &                    & \L \succ \0.
 \end{array}
\end{equation*}
$\L$ is the decision variable.
$\MVEE(\SC)$ corresponds to the formulation of computing 
the minimum volume enclosing ellipsoid (MVEE) centered at the origin for $\SC$.
It has been shown in \cite{Miz14, Gil15} that the optimal solution $\L^*$
is given by $(\F\F^\top)^{-1}$.
Therefore, $(\L^*)^{1/2}$ can be used for the preconditioning matrix in order to 
improve the condition number of $\F$,
since $\kappa((\L^*)^{1/2} \F )^2 = \kappa(\F^\top \L^* \F) = \kappa(\I) = 1$.
Next, we consider the noisy case on $\wt{\A}$.
Let $\L^*$ be the optimal solution of $\MVEE(\SC)$
for $\SC = \{\wt{\a}_1, \ldots, \wt{\a}_m\}$ where $\wt{\a}_1, \ldots, \wt{\a}_m$ 
are the column vectors of $\wt{\A}$.
In this case, $\L^*$ does not completely match $(\F\F^\top)^{-1}$, 
but the difference between these two matrices is thought to be small 
if the amount of noise is also small.
Therefore, $(\L^*)^{1/2}$ could serve as a preconditioning matrix 
for restricting the condition number of $\F$.

We may need to add a further explanation of $\MVEE(\SC)$.
The origin-centered MVEE for the points in $\{\pm \a : \a \in \SC \}$ 
is given as $\{ \x \in \Real^r: \x^\top \L^* \x \le 1\}$
where $\L^*$ is the optimal solution.
The volume of the MVEE is $c(r) / \sqrt{\det (\L^*)}$ where $c(r)$ is the volume of 
a unit ball in $\Real^r$ and a real number depending on the dimension $r$.
Since $\mbox{rank}(\A) = r$  due to Assumption \ref{Assump}(a),
the convex hull of the points in $\{\pm \a : \a \in \SC \}$ is full-dimensional in $\Real^r$,
and thus, the MVEE has a positive volume.
$\MVEE(\SC)$ is a convex optimization problem.
Efficient algorithms such as 
interior-point algorithms and the Frank-Wolfe algorithms have been developed
and are now available for solving it; see, for instance, \cite{Kha96, Sun04} 
for the details on these algorithms.

Gillis and Vavasis showed in \cite{Gil15} that 
the preconditioner $(\L^*)^{1/2}$ makes it possible to improve the noise robustness of SPA
under Assumption \ref{Assump} and $d=r$.
Here is their result.

\begin{theo}[Theorem 2.9 of \cite{Gil15}] \label{Theo: robustness of precond SPA}
 Let $\wt{\A} = \A + \N$ for $\A \in \Real^{d \times m}$ and $\N \in \Real^{d \times m}$.
 Suppose that $\A$ satisfies Assumption \ref{Assump} and the condition $d=r$.
 Let $\L^*$  be the optimal solution of $\MVEE(\SC)$ 
 where $\SC = \{\wt{\a}_1, \ldots, \wt{\a}_m\}$ for $\wt{\a}_1, \ldots, \wt{\a}_m$ of $\wt{\A}$.
 If $\n_i$ of $\N$ satisfies
 $\|\n_i\|_2 \le \epsilon$ for all $i = 1, \ldots, m$ with
 \begin{equation*}
  \epsilon \le O\Biggl(\frac{\sigma_{\mmin}(\F)}{r\sqrt{r}}\Biggr),
 \end{equation*}
 then, Algorithm \ref{Alg: SPA} with the input $((\L^*)^{1/2} \wt{\A}, r)$ returns
 the output $\IC$ such that 
 the size of the basis error of $\IC$ is up to $O(\kappa(\F) \epsilon)$.
\end{theo}

The ``size of the basis error'' in the above statement
should be clearly explained.
Let $\IC$ be a subset of $\{1, \ldots, m\}$ with $r$ elements, 
and suppose the elements are arranged in some order.
Given a near-separable matrix $\wt{\A}$ of (\ref{Eq: near-separable matrix}),
the size of the basis error of $\IC$ is 
\begin{equation*}
 \max_{j = 1, \ldots, r} \|\wt{\a}_{\IC(j)} - \f_j \|_2.
\end{equation*}

Gillis and Ma in \cite{Gil15b} developed 
other types of preconditioning matrices for SPA and analyzed the noise robustness.

\subsubsection{Case of $d \neq r$} \label{Subsubsec: d!=r}

The discussion in the previous section was made 
under Assumption \ref{Assump} and the condition $d=r$.
We shall consider the usual situation in which
a near-separable matrix $\wt{\A}$ of (\ref{Eq: near-separable matrix}) has $d \neq r$.
The following approach was suggested in \cite{Gil15} for handling this situation.
SVD plays a key role.
SVD decomposes $\wt{\A}$ into $\wt{\A} = \U\Sigmab\V^\top$ 
where $\U$, $\V$ and $\Sigmab$ are the same as those of (\ref{Eq: SVD}).
By using $\Sigmab$, we construct a \by{d}{m} diagonal matrix $\Sigmab^r$ such that
\begin{equation} \label{Eq: truncated sigma matrix}
\Sigmab^r =
 \left\{
\begin{array}{ll}
 (\mbox{diag}(\sigma_1, \ldots, \sigma_r, 0, \ldots, 0), \0)  & \mbox{if} \ d \le m, \\
 (\mbox{diag}(\sigma_1, \ldots, \sigma_r, 0, \ldots, 0); \0)  & \mbox{otherwise}.
\end{array} 
\right.
\end{equation}
Let $\A^r = \U \Sigmab^r \V^\top$.
This is the best rank-$r$ approximation matrix to $\A$ 
under the matrix $2$-norm or the Frobenius norm. 
Also, note that $\wt{\A} = \A^r$ holds if $\wt{\A}$ does not contain $\N$.
We construct a matrix $\P \in \Real^{r \times m}$ such that 
\begin{equation} \label{Eq: relation P and rA}
 \left(
 \begin{array}{c}
  \P \\
  \0
 \end{array}
\right) 
 = \U^\top \A^r, \ \mbox{equivalently}, \ 
 \P = \mbox{diag}(\sigma_1, \ldots, \sigma_r) (\V^r)^\top
\end{equation}
where $\V^r = (\v_1, \ldots, \v_r) \in \Real^{m \times r}$ 
for the column vectors $\v_1, \ldots, \v_r$ of $\V$.
As we will see in Section \ref{Subsec: preliminaries}, 
$\P$ can be thought of as an \by{r}{m} near-separable matrix having an \by{r}{r} basis matrix.
Therefore, we can apply the discussion in the previous section to $\P$.
Let $\SC = \{\p_1, \ldots, \p_m\}$ for $\P$.
We compute the optimal solution $\L^*$ of $\MVEE(\SC)$
and run Algorithm \ref{Alg: SPA} on $((\L^{*})^{1/2} \P, r)$.
Algorithm \ref{Alg: precond SPA} summarizes each step of the preconditioned SPA.
The description is almost the same as that of Algorithm 2 of \cite{Gil15}.

\begin{algorithm}
 \caption{Preconditioned SPA}
 \label{Alg: precond SPA}
 \textbf{Input:} 
 A \by{d}{m} real matrix $\A$ and a positive integer $r$. \\
 \textbf{Output:} 
 An index set $\IC$.
 \begin{enumerate}
  \item[\textbf{1:}] 
	       Compute the SVD of $\A$.
	       Let $\sigma_1, \ldots, \sigma_r$ be the top $r$ largest singular values, 
	       and $\v_1, \ldots, \v_r \in \Real^m$ be the corresponding right singular vectors.
	       Construct $\P = \mbox{diag}(\sigma_1, \ldots, \sigma_r) (\V^r)^\top \in \Real^{r \times m}$ 
	       for $\V^r = (\v_1, \ldots, \v_r)$.

  \item[\textbf{2:}] 
	       Compute the optimal solution $\L^*$ of $\MVEE(\SC)$
	       for $\SC = \{\p_1, \ldots, \p_m\}$ where
	       $\p_1, \ldots, \p_m$ are the column vectors of $\P$.

  \item[\textbf{3:}] 
	       Construct $\P^\circ = (\L^*)^{1/2}\P$.
	       Run Algorithm \ref{Alg: SPA} on the input $(\P^\circ, r)$, 
	       and output the index set $\IC$ obtained by the algorithm.

 \end{enumerate}
\end{algorithm}

\subsection{Our Result and Its Comparison with 
  Theorem \ref{Theo: robustness of precond SPA} by Gillis and Vavasis}

Gillis and Vavasis in \cite{Gil15} showed empirical results, 
suggesting that Algorithm \ref{Alg: precond SPA} can improve the noise robustness of SPA.
However, a formal analysis was not given.
Here, we give it. 

\begin{theo} \label{Theo: main}
 Let $\wt{\A} = \A + \N$ for $\A \in \Real^{d \times m}$ and $\N \in \Real^{d \times m}$.
 Suppose that $r \ge 2$ and $\A$ satisfies Assumption \ref{Assump}.
 If $\N$ satisfies $\|\N\|_2 = \epsilon$ with
 \begin{equation*}
  \epsilon \le \frac{\sigma_{\mmin}({\F})}{1225\sqrt{r}},
 \end{equation*}
 then, Algorithm \ref{Alg: precond SPA} with the input $(\wt{\A}, r)$ returns
 the output $\IC$ such that
 there is an order of the elements in $\IC$ satisfying
 \begin{equation*}
  \|\wt{\a}_{\IC(j)} - \f_j\|_2  \le (432 \kappa(\F) + 4) \epsilon
 \end{equation*}
 for all $j = 1, \ldots, r$.
\end{theo}

The proof is given in Section \ref{Sec: analysis}.
Note that in this paper 
we will use the notation $\epsilon$ to describe the size of $\|\N\|_2$ or $\|\n_i\|_2$
for the noise matrix $\N$ of a near-separable one.
Let us compare our result, Theorem \ref{Theo: main}, 
with that of Gillis and Vavasis, Theorem \ref{Theo: robustness of precond SPA}.
The main advantage of ours is in that it ensures the noise robustness of 
the preconditioned SPA for separable NMF problems without the condition $d = r$, 
while their result ensures it under that condition.
The dimension $d$ is usually greater than the factorization rank $r$
in separable NMF problems derived from actual applications
such as topic extraction from documents \cite{Aro12b, Aro13, Miz14} 
and endmember detection in hyperspectral images \cite{Gil14, Gil15}.
Therefore, our result can be used as a guide for seeing how robust the preconditioned SPA is to noise 
when handling such applications.

As we will see in Section \ref{Sec: analysis},
$\P$ in step 1 of Algorithm \ref{Alg: precond SPA}
is an \by{r}{m} near-separable matrix having an \by{r}{r} basis matrix.
Furthermore, Assumption \ref{Assump} holds for the basis matrix of $\P$, 
if the amount of noise involved in an input matrix is small.
Therefore, Theorem \ref{Theo: robustness of precond SPA} can apply to $\P$, 
and this implies a similar result to ours.

\begin{prop} \label{Prop: robustness of precond SPA with d!=r}
 Let $\wt{\A} = \A + \N$ for $\A \in \Real^{d \times m}$ and $\N \in \Real^{d \times m}$.
 Suppose that Assumption \ref{Assump} holds for $\A$ in $\wt{\A}$.
 If $\N$ satisfies $\|\N\|_2 = \epsilon$ with
 \begin{equation*}
  \epsilon \le O\Biggl( \frac{\sigma_{\mmin}({\F})}{r\sqrt{r}} \Biggr),
 \end{equation*}
 then, Algorithm \ref{Alg: precond SPA} with the input $(\wt{\A}, r)$ returns
 the output $\IC$ such that 
 the size of the basis error of $\IC$ is up to $O(\kappa(\F)\epsilon)$.
\end{prop}
The proof is given in Section \ref{Sec: analysis}.
Although the size of the basis error in Proposition \ref{Prop: robustness of precond SPA with d!=r} 
is of the same order as ours,
the allowable amount of noise is worse by a factor $1/r$.

Our allowed noise range is described  using 
the norm of $\N$, while theirs is described by the norm of the column vectors $\n_i$.
Hence, we shall rewrite our result in terms of the norm of $\n_i$.
By taking account of the fact that 
$\|\N\|_2 \le \sqrt{m} \max_{i= 1, \ldots, m} \|\n_i\|_2$ for $\N \in \Real^{d \times m}$,
Theorem \ref{Theo: main} implies the following corollary.
\begin{coro} \label{Coro: main theorem}
 Let $\wt{\A} = \A + \N$ for $\A \in \Real^{d \times m}$ and $\N \in \Real^{d \times m}$.
 Suppose the same conditions in Theorem \ref{Theo: main} hold.
 If $\n_i$ of $\N$ satisfies
 $\|\n_i\|_2 \le \epsilon$ for all $i = 1, \ldots, m$ with
 \begin{equation*}
  \epsilon \le O\Biggl(\frac{\sigma_{\mmin}(\F)}{\sqrt{rm}}\Biggr),
 \end{equation*}
 then, Algorithm \ref{Alg: precond SPA} with the input $(\wt{\A}, r)$ 
 returns the output $\IC$ such that 
 the size of the basis error of $\IC$ is up to $O(\kappa(\F) \epsilon)$.
\end{coro}
We see that $1 / \sqrt{m}$  emerges in the description of allowed noise range.
When handling separable NMF problems from actual applications,
it would be reasonable to suppose a situation where $m$ corresponds to 
the number of data points, and it could be large.
Therefore, the allowed noise range of our result becomes weaker in such a situation.

\section{Analysis of the Noise Robustness of Algorithm \ref{Alg: precond SPA}} 
\label{Sec: analysis}

The main goal of this section is to prove Theorem \ref{Theo: main}.
In the discussion of the proof, we will see that Theorem \ref{Theo: robustness of precond SPA} 
implies Proposition \ref{Prop: robustness of precond SPA with d!=r}.

\subsection{Preliminaries}
\label{Subsec: preliminaries}

Let $\wt{\A}$ be of the form given in (\ref{Eq: near-separable matrix}).
We shall consider Algorithm \ref{Alg: precond SPA} on the input data $(\wt{\A}, r)$.
Step 1 computes the SVD of $\wt{\A}$, and decomposes it into $\wt{\A} = \U\Sigmab\V^\top$
where $\U$, $\V$ and $\Sigmab$ are the same as those of (\ref{Eq: SVD}).
The rank-$r$ approximation matrix $\wt{\A}^r$
is given as $\wt{\A}^r = \U \Sigmab^r \V^\top$ by using $\Sigmab^r$ of (\ref{Eq: truncated sigma matrix}).
We denote $\wt{\A} - \wt{\A}^{r}$ by $\wt{\A}^{r,c}$.
Then, $\wt{\A}$ can be represented as 
\begin{equation} \label{Eq: expansion of wtA}
 \wt{\A} = \wt{\A}^{r} + \wt{\A}^{r,c}
\end{equation}
by using $\wt{\A}^{r}$ and $\wt{\A}^{r,c}$ such that 
\begin{equation*}
 \wt{\A}^r = \U \Sigmab^r \V^\top 
  \ \mbox{and} \
  \wt{\A}^{r,c} = \U \Sigmab^{r,c} \V^\top.
\end{equation*}
Here, we let $\Sigmab^{r,c} = \Sigmab - \Sigmab^r$.

$\P$ in step 1 of Algorithm \ref{Alg: precond SPA} is given as (\ref{Eq: relation P and rA}).
Using relation (\ref{Eq: expansion of wtA}), 
it can be rewritten as 
\begin{eqnarray}
(\P; \0)
 &=& \U^\top \wt{\A}^r  \label{Eq: P and wtA^r} \\
 &=& \U^\top(\wt{\A} - \wt{\A}^{r,c}) \nonumber \\
 &=& \U^\top(\A + \N - \wt{\A}^{r,c}) \nonumber \\
 &=& \U^\top(\A + \wb{\N}) \nonumber \\
 &=& \U^\top( (\F,\F\K)\Pib + \wb{\N}) \nonumber \\
 &=& \U^\top( \F + \wb{\N}^{(1)}, \F\K + \wb{\N}^{(2)})\Pib \nonumber \\
 &=& \U^\top( \wh{\F}, \wh{\F}\K + \wh{\N})\Pib.  \nonumber
\end{eqnarray}
In the above, we have used the notation $\wb{\N} \in \Real^{d \times m}$, 
$\wb{\N}^{(1)} \in \Real^{d \times r}$, 
$\wb{\N}^{(2)} \in \Real^{d \times (m-r)}$, 
$\wh{\F} \in \Real^{d \times r}$, 
and $\wh{\N} \in \Real^{d \times (m-r)}$ such that 
\begin{eqnarray}
\wb{\N} &=& \N - \wt{\A}^{r,c}, \label{Eq: wbN} \\
(\wb{\N}^{(1)}, \wb{\N}^{(2)}) &=& \wb{\N} \Pib^{-1}, \nonumber \\
\wh{\F} &=& \F + \wb{\N}^{(1)}, \label{Eq: whF} \\
\wh{\N} &=& -\wb{\N}^{(1)}\K + \wb{\N}^{(2)}. \label{Eq: whN}
\end{eqnarray}
Accordingly, $\P$ can be represented as 
\begin{equation} \label{Eq: P}
 \P = (\G, \G\K+\S) \Pib
\end{equation}
by using $\G \in \Real^{r \times r}$ and $\S \in \Real^{r \times (m-r)}$ 
such that
\begin{equation} \label{Eq: whF and whN}
 \U^\top\wh{\F} = 
 \left(
 \begin{array}{c}
  \G \\ 
  \0
 \end{array}
\right)
\ \mbox{and} \ 
 \U^\top\wh{\N} =
 \left(
 \begin{array}{c}
  \S \\ 
  \0
 \end{array}
\right).
\end{equation}
$\P^\circ$ in the step 3 of Algorithm \ref{Alg: precond SPA} is
\begin{equation*} 
 \P^\circ = (\G^\circ, \G^\circ \K+\S^\circ) \Pib
\end{equation*}
by using $\G^\circ \in \Real^{r \times r}$ and $\S^\circ \in \Real^{r \times (m-r)}$ such that 
\begin{equation*}
 \G^\circ = (\L^*)^{1/2} \G 
  \ \mbox{and} \ 
 \S^\circ = (\L^*)^{1/2} \S. 
\end{equation*}
Hence, $\P$ and $\P^\circ$ are near-separable matrices, and 
$(\G, \G\K)\Pib$ and $(\G^\circ, \G^\circ\K)\Pib$ correspond to the separable matrices.
In particular, $\G$ and $\G^\circ$ are the basis matrices,
and $(\0, \S) \Pib$ and $(\0, \S^\circ) \Pib$ are the noise matrices.
It should be noted that $\P$ and $\P^\circ$ are \by{r}{m} near-separable matrices 
and these have \by{r}{r} basis matrices.

\subsection{Proof of Proposition  \ref{Prop: robustness of precond SPA with d!=r}}
\label{Subsubsec: Proof of proposition}

Here, we prove several lemmas that will be necessary for the subsequent discussion.
Similar statements have already been proven in \cite{Miz14}.
More precisely, Lemmas \ref{Lemm: bound for barN}, \ref{Lemm: bound for s} and 
\ref{Lemm: singular values of F and G} correspond to (a), (b) and (c) of Lemma 7 of that paper,
and we have included them here to make the discussion self-contained.
After that, we prove Proposition  \ref{Prop: robustness of precond SPA with d!=r}.
The proof is obtained from Theorem \ref{Theo: robustness of precond SPA} 
together with the following lemmas.

\begin{lemm} \label{Lemm: singular value of perturbed matrix}
 Let $\wt{\A} = \A + \N \in \Real^{d \times m}$.
 Then, $|\sigma_i(\wt{\A}) - \sigma_i(\A)| \le \|\N\|_2$ for each $i=1, \ldots, t$ where $t = \min(d,m)$.
 \end{lemm}
\begin{proof}
See Corollary 8.6.2 of \cite{Gol96}.
\end{proof}

\begin{lemm}[Lemma 7(a) of \cite{Miz14}] \label{Lemm: bound for barN}
 $\|\wb{\N}\|_2 \le 2 \|\N\|_2$.
\end{lemm}
\begin{proof}
 $\|\wb{\N}\|_2 = \|\N - \wt{\A}^{r,c}\|_2  \le \|\N\|_2 + \|\wt{\A}^{r,c}\|_2$
 since $\wb{\N}$ is of the form (\ref{Eq: wbN}).
 Also, $\|\wt{\A}^{r,c}\|_2 = \sigma_{\mmax}(\wt{\A}^{r,c}) = \sigma_{r+1}(\wt{\A})$.
 From Lemma \ref{Lemm: singular value of perturbed matrix} and $\sigma_{r+1}(\A) = 0$,
 we have $\sigma_{r+1}(\wt{\A}) \le  \|\N\|_2$.
 Thus, $\|\wb{\N}\|_2 \le 2 \|\N\|_2$.
\end{proof}

\begin{lemm}[Lemma 7(b) of \cite{Miz14}] \label{Lemm: bound for s}
 Let $\s_i$ be the column vector of $\S$ in (\ref{Eq: P}).
 Suppose that $\K$ satisfies Assumption \ref{Assump}(b). Then,
 $\|\s_i\|_2 \le 4\|\N\|_2 $ for each $i = 1, \ldots, m-r$.
\end{lemm}
\begin{proof}
 We see from (\ref{Eq: whF and whN}) and (\ref{Eq: whN}) that 
 $(\s_i; \0) = \U^\top \wh{\n}_i$ and $\wh{\n}_i = - \wb{\N}^{(1)} \k_i + \wb{\n}_i^{(2)}$.
 Here, $\k_i$ and $\wb{\n}_i^{(2)}$ are the column vectors of $\K$ and $\wb{\N}^{(2)}$, 
 respectively.
 Thus, 
 $\|\s_i\|_2 
   =
  \|- \wb{\N}^{(1)} \k_i + \wb{\n}_i^{(2)}\|_2  
   \le
  \|\wb{\N}^{(1)}\|_2 \|\k_i\|_2 + \|\wb{\n}_i^{(2)}\|_2$.
 From Lemma \ref{Lemm: bound for barN} and $\|\k_i\|_2 \le 1$ due to Assumption \ref{Assump}(b), 
 we have $\|\s_i\|_2 \le 4\|\N\|_2$.
\end{proof}

\begin{lemm}[Lemma 7(c) of \cite{Miz14}] \label{Lemm: singular values of F and G}
 $|\sigma_j(\F) - \sigma_j(\G)| \le 2\|\N\|_2$ for each $j = 1, \ldots, r$.
\end{lemm}
\begin{proof}
 We see from (\ref{Eq: whF and whN}) that $\G$ and $\wh{\F}$ has the relation 
 $(\G; \0) = \U^\top \wh{\F}$.  Since $\U$ is an orthogonal matrix,
 the singular values of $\G$ coincide with those of $\wh{\F}$.
 Also, $\wh{\F}$ is of the form (\ref{Eq: whF}).
 Thus, from Lemma \ref{Lemm: singular value of perturbed matrix}, 
 we have 
 $
  |\sigma_j(\F) - \sigma_j(\G)| = |\sigma_j(\F) - \sigma_j(\F + \wb{\N}^{(1)}) | 
   \le \|\wb{\N}^{(1)}\|_2 \le 2\|\N\|_2.
 $
 The last inequality follows from Lemma \ref{Lemm: bound for barN}.
\end{proof}

\begin{lemm} \label{Lemm: relation of basis errors in original and transformed space}
 Let $\wt{\a}_k$ and $\p_k$ be the column vectors of  $\wt{\A}$ and $\P$.
 Also, let $\f_j$ and $\g_j$ be those of  $\F$ and $\G$.
 We have $\|\wt{\a}_k - \f_j \|_2 \le  \|\p_k - \g_j \|_2 + 3\|\N\|_2$
 for any $k$ and any $j$ in $\{1, \ldots, r\}$.
\end{lemm}
\begin{proof}
 We see from (\ref{Eq: P and wtA^r}) and (\ref{Eq: expansion of wtA}) that 
 $\wt{\a}_k$ and $\p_k$ are related as 
 \begin{equation*}
  \left(
  \begin{array}{c}
   \p_k \\
   \0
  \end{array}
  \right)  = \U^\top \wt{\a}^r_k
  \ \mbox{and} \
 \wt{\a}_k = \wt{\a}^r_k + \wt{\a}^{r,c}_k
 \end{equation*}
 where $\wt{\a}^r_k$ and $\wt{\a}^{r,c}_k$ are the column vectors of
 $\wt{\A}^r$ and  $\wt{\A}^{r,c}$.
 Also, from (\ref{Eq: whF and whN}) and (\ref{Eq: whF}), $\g_j$ and $\f_j$ 
 are related as 
 \begin{equation*}
  \left(
  \begin{array}{c}
   \g_j \\
   \0
  \end{array}
  \right)  = \U^\top \wh{\f}_j
  \ \mbox{and} \
  \wh{\f}_j = \f_j + \wb{\n}^{(1)}_j
 \end{equation*}
 where $\wh{\f}_j$ and $\wb{\n}^{(1)}_j$ are the column vectors of 
 $\wh{\F}$ and $\wb{\N}^{(1)}$.
 Therefore, 
 \begin{eqnarray*}
  \|\wt{\a}_k - \f_j \|_2 
   &=& 
   \|(\wt{\a}_k^{r} + \wt{\a}_k^{r,c}) 
   - (\wh{\f}_j - \wb{\n}_j^{(1)})\|_2  \\
  &=& 
   \|\U^\top (\wt{\a}_k^{r} - \wh{\f}_j) 
   + \U^\top (\wt{\a}_k^{r,c} + \wb{\n}_j^{(1)}) \|_2  \\
  &\le&
      \|\U^\top (\wt{\a}_k^{r} - \wh{\f}_j)\|_2 
      + \|\U^\top \wt{\a}_k^{r,c}\|_2 +  \|\U^\top\wb{\n}_j^{(1)}\|_2  \\
  &=& 
   \|\p_k - \g_j\|_2  + \|\wt{\a}_k^{r,c}\|_2 + \|\wb{\n}_j^{(1)}\|_2. 
 \end{eqnarray*}
 $\U$ is a \by{d}{d} orthogonal matrix in  (\ref{Eq: P and wtA^r}) that 
 consists of the left singular vectors of $\wt{\A}$.
 By Lemma \ref{Lemm: bound for barN}, 
 we can put an upper bound on the norm of  $\wb{\n}_j^{(1)}$ such that 
 $\|\wb{\n}_j^{(1)}\|_2 \le \|\wb{\N}\|_2 \le 2\|\N\|_2$.
 Also, we can put an upper bound on the norm of $\wt{\a}_k^{r,c}$
 such that $\|\wt{\a}_k^{r,c}\|_2 \le \|\N\|_2$
 due to $\|\wt{\a}_k^{r,c}\|_2 \le \|\wt{\A}^{r,c}\|_2 \le \|\N\|_2$.
 The last inequality is obtained in the same way as in the proof of
 Lemma \ref{Lemm: bound for barN}.
 Therefore, we have
 $\|\wt{\a}_k - \f_j \|_2 \le  \|\p_k - \g_j \|_2 + 3\|\N\|_2$.
\end{proof}

We are now ready to prove Proposition \ref{Prop: robustness of precond SPA with d!=r}.
\begin{proof}[(Proof of Proposion \ref{Prop: robustness of precond SPA with d!=r})]  
 We show that Theorem \ref{Theo: robustness of precond SPA} can apply to 
 $\P$ in step 1 of  Algorithm \ref{Alg: precond SPA}, 
 if the amount of noise $\|\N\|_2$ is smaller than some level.
 $\P$ can be written as (\ref{Eq: P}), and hence, is 
 an \by{r}{m} near-separable matrix having an \by{r}{r} basis matrix $\G$.
 We choose some real number $\gamma$ such that $\gamma > 2$. 
 Suppose that $\|\N\|_2  \le \frac{1}{\gamma} \sigma_{\mmin}(\F)$. 
 It follows from Lemma \ref{Lemm: singular values of F and G} and Assumption \ref{Assump}(a)
 that $\sigma_{\mmin}(\G) \ge \frac{\gamma-2}{\gamma} \sigma_{\mmin}(\F) > 0$.
 Therefore, Assumption \ref{Assump} holds for $\G$. 
 This means that  Theorem \ref{Theo: robustness of precond SPA} can apply to $\P$.
 Its application leads to the following statement. 
 Let $\L^*$ be the optimal solution of $\MVEE(\SC)$ where $\SC = \{\p_1, \ldots, \p_m\}$ for $\P$.
 If $\s_i$ of $\S$ satisfies $\|\s_i\|_2 \le \epsilon$ for all $i = 1, \ldots, m-r$
 with $\epsilon \le O ( \sigma_{\mmin}(\G) / r \sqrt{r} )$,
 then, Algorithm \ref{Alg: SPA} with $((\L^*)^{1/2} \P, r)$ returns the output $\IC$ such that 
 there is an order of the elements in $\IC$ satisfying
 $\|\p_{\IC(j)} - \g_j \|_{2} \le O(\kappa(\G)\epsilon)$ for $j=1, \ldots, r$.

 Suppose that 
 \begin{equation} \label{Eq: size of n}
  \|\N\|_2 \le \frac{\sigma_{\mmin}(\F)}{4 \alpha r \sqrt{r} +2 }
 \end{equation}
 for some $\alpha \ge 1$.
 Note also that $\|\N\|_2 < \frac{1}{2} \sigma_{\mmin}(\F)$.
 From Lemmas \ref{Lemm: bound for s} and \ref{Lemm: singular values of F and G},
 $\|\s_i\|_2 \le 4\|\N\|_2$ and 
 $\frac{\sigma_{\mmin}(\F) - 2 \|\N\|_2}{\alpha r \sqrt{r}} \le  \frac{\sigma_{\mmin}(\G)}{\alpha r \sqrt{r}}$.
 Since $\|\N\|_2$ is supposed to satisfy (\ref{Eq: size of n}), we have
 \begin{equation*}
  \|\s_i\|_2 \le 4\|\N\|_2   
   \le \frac{\sigma_{\mmin}(\F) - 2\|\N\|_2}{\alpha r \sqrt{r}} \le \frac{\sigma_{\mmin}(\G)}{\alpha r \sqrt{r}}.
 \end{equation*}
 This gives $\|\s_i\|_2 \le \frac{\sigma_{\mmin}(\G)}{\alpha r \sqrt{r}}$.
 Also, from Lemma \ref{Lemm: singular values of F and G},  
 \begin{eqnarray*}
  \kappa(\G) 
  &=& \frac{\sigma_{\mmax}(\G)}{\sigma_{\mmin}(\G)} \\
  &\le& \frac{\sigma_{\mmax}(\F) + 2\|\N\|_2}{\sigma_{\mmin}(\F) - 2\|\N\|_2} \\
  &=& \biggl(1 + \frac{1}{2\alpha r \sqrt{r}}\biggr) \kappa(\F) + \frac{1}{2\alpha r \sqrt{r}} \\
  &\le& 2 \kappa(\F) + 1.
 \end{eqnarray*}
 Therefore, from Lemmas \ref{Lemm: bound for s} and \ref{Lemm: relation of basis errors in original and transformed space},
 we have 
 \begin{equation*}
  \|\wt{\a}_{\IC(j)} - \f_j \|_{2} - 3\|\N\|_2
   \le 
   \|\p_{\IC(j)} - \g_j \|_{2}
   \le
   \beta \kappa(\G) \|\s_{i^*}\|_2
   \le
   (8\beta\kappa(\F) + 4\beta)\|\N\|_2
 \end{equation*}
 for some $\beta > 0$ where let $i^* = \arg \max_{i = 1, \ldots, m-r} \|\s_i\|_{2}$.
 Consequently, if $\|\N\|_2$ satisfies (\ref{Eq: size of n}),
 then, Theorem \ref{Theo: robustness of precond SPA} can apply to $\P$, 
 and the application implies that 
 $\|\wt{\a}_{\IC(j)} - \f_j \|_{2} \le (8\beta\kappa(\F) + 4\beta  + 3 )\|\N\|_2$.

\end{proof}

\subsection{Proof of Theorem \ref{Theo: main}}

The core part of the proof of Theorem \ref{Theo: main}
is to show that Theorem \ref{Theo: robustness of SPA} can apply to $\P^\circ$ 
if $\|\N\|_2$ is small.
To do this, we evaluate the upper bound 
on the condition number of $\G^\circ$.
For the subsequent discussion, we need the following lemma.

\begin{lemm} \label{Lemm: inequalities about F, G and s}
 Let $\alpha$ be a real constant satisfying $\alpha > \sqrt{2}$ and 
 $r$ be any integer satisfying $r \ge 2$.
 Suppose that 
 $\|\N\|_2 \le \frac{\sigma_{\mmin}({\F})}{\alpha \sqrt{r}}$ and
 $\K$ satisfies Assumption \ref{Assump}(b). Then, 
 \begin{enumerate}[{\normalfont (a)}]
  \item 
	$
	\frac{\alpha\sqrt{r}-2}{\alpha\sqrt{r}} \sigma_{\mmin}(\F) 
	\le \sigma_{\mmin}(\G)
	\le \sigma_{\mmax}(\G) 
	\le \sigma_{\mmax}(\F) + \frac{2}{\alpha\sqrt{r}} \sigma_{\mmin}(\F).
	$
  \item
       $\|\s_i\|_2 \le \frac{4\sigma_{\mmin}(\G)}{\alpha \sqrt{r}-2}$
       for $i = 1, \ldots, m-r$.
 \end{enumerate}
\end{lemm}
\begin{proof}
 Statement (a) follows from Lemma \ref{Lemm: singular values of F and G}, and 
 that (b) follows from Lemmas \ref{Lemm: bound for s} and (a).
\end{proof}

\subsubsection{Upper Bound on the Condition Number of $\G^\circ$}
\label{Subsec: bound for the cond of cirG}

We show that the condition number of $\G^\circ$ is bounded from above 
by a real constant, if $\|\N\|_2$ is smaller than some level.
This can be proved by applying Theorem 2.8 in \cite{Gil15} to $\P$
and by taking into account the discussion 
in the proof of Proposition \ref{Prop: robustness of precond SPA with d!=r}.
This implies that such an upper bound is obtained 
under $\|\N\|_2 \le O(\sigma_{\mmin}(\F)  / r \sqrt{r})$.
Meanwhile, we will provide a similar bound under $\|\N\|_2 \le O(\sigma_{\mmin}(\F) / \sqrt{r})$.
This result can increase the allowable amount of noise $\|\N\|_2$ by a factor $1/r$
over that of the application of the theorem.

Our derivation follows that of \cite{Gil15}.
More precisely, it will be shown through 
Lemmas \ref{Lemm: lower bound for C} and \ref{Lemm: upper bound for C}
and Proposition \ref{Prop: Bound for eigenvalues},
which have the following correspondence to the lemmas of that paper;
Lemma \ref{Lemm: lower bound for C} is Lemma 2.4,
Lemma \ref{Lemm: upper bound for C} is Lemmas 2.5 and 2.6, 
and Proposition \ref{Prop: Bound for eigenvalues} is Lemmas 2.6 and 2.7.
Although a major part of our proof of the lemmas and proposition 
relies on the techniques developed in that paper, 
there are some differences.
In particular, we use an alternate technique to prove Proposition \ref{Prop: Bound for eigenvalues},
and it allows us to derive the upper bound on the condition number of $\G^\circ$ 
under $\|\N\|_2 \le O(\sigma_{\mmin}(\F) / \sqrt{r})$.
Remark \ref{Remark: difference in proof} details the differences between the proofs.

Now let us look at problem $\MVEE(\SC)$ in step 2 of Algorithm \ref{Alg: precond SPA}.
Note that $\SC = \{\p_1, \ldots, \p_m\}$ for the column vectors $\p_1, \ldots, \p_m$ of $\P$. 
For convenience,
we will change the variable to $\MVEE(\SC)$ 
such that $\C = \G^\top \L \G$ for a nonsingular matrix $\G$ and 
consider the problem, 
\begin{equation*} 
 \begin{array}{lll}
  \transMVEE(\SC): 
   & \mbox{minimize}   
   & -\log \det(\C) + 2 \log \det(\G), \\
   & \mbox{subject to}  
   & \p^\top (\G^{-1})^\top \C \G^{-1} \p \le 1
   \ \mbox{for all} \ \p \in \SC, \\
   &                    & \C \succ \0,
 \end{array}
\end{equation*}
which is equivalent to $\MVEE(\SC)$ under the nonsingular transformation $\G$.
$\C$ is the decision variable.
Let $\C^*$ be the optimal solution and 
$\lambda_j$ denote the $j$th eigenvalue of $\C^*$.
We will continue to use the notation $\lambda_j$ for this purpose 
throughout this section.
For the $j$th singular value $\sigma_j$ of $\G^\circ$,
we have $\sigma_j = \lambda_j^{1/2}$ for $j = 1, \ldots, r$
since $(\G^\circ)^\top \G^\circ = \G^\top \L^* \G = \C^*$.

In Lemmas \ref{Lemm: lower bound for C} and \ref{Lemm: upper bound for C},
we evaluate the lower and upper bounds on $\det (\C^*)$ by using $r$ and $\lambda_j$.
These bounds give the inequality that $r$ and $\lambda_j$ need to satisfy.
In Proposition \ref{Prop: Bound for eigenvalues}, 
by using the inequality, we derive the lower and upper bounds on $\lambda_j$ 
whose square root is equal to the singular value $\sigma_j$ of $\G^\circ$.
Lemma \ref{Lemm: lower bound for C} is almost the same as Lemma 2.4 of \cite{Gil15}.

\begin{lemm} \label{Lemm: lower bound for C}
 Let $\alpha$ be a real constant satisfying $\alpha > \sqrt{2}$ and 
 $r$ be any integer satisfying $r \ge 2$.
 Suppose that $\|\N\|_2 \le \frac{\sigma_{\mmin}({\F})}{\alpha \sqrt{r}}$ 
 and $\K$ satisfies Assumption \ref{Assump}(b).
 Then, 
 \begin{equation*}
 \det(\C^*) \ge \Biggl(\frac{\alpha \sqrt{r} - 2}{\alpha \sqrt{r} + 2} \Biggr)^{2r}.  
 \end{equation*}
\end{lemm}
\begin{proof}
 For an \by{r}{r} scaled identity matrix $\theta \I$ with a positive real number $\theta$,
 we derive the upper bound on $\theta$ such that 
 $\theta \I$ is feasible for $\transMVEE(\SC)$.
 Since $\P$ can be written as (\ref{Eq: P}),
 $\SC$ contains two different types of vectors: one 
 is $\g_j$ and the other is $\G\k_i + \s_i$
 where $\g_j$ is the column vector of $\G$ 
 and $\k_i$ and $\s_i$ are those of $\K$ and $\S$.
 Therefore, $\theta \I$ needs to satisfy two types of constraints,
 \begin{eqnarray*}
  & & \theta \g_j^\top (\G^{-1})^\top \I \G^{-1} \g_j \le 1  
   \ \mbox{for} \ j = 1, \ldots, r,   \\
  & & \theta (\G\k_i + \s_i)^\top (\G^{-1})^\top \I \G^{-1} (\G\k_i + \s_i) \le 1  
   \ \mbox{for} \ i = 1, \ldots, m-r.
\end{eqnarray*}
 The first constraints hold if $\theta \le 1$.
 For the second constraints, we have 
 \begin{eqnarray*}
  (\G\k_i + \s_i)^\top (\G^{-1})^\top \I \G^{-1} (\G\k_i + \s_i) 
  &=& \|\k_i + \G^{-1}\s_i\|_2^2 \\
  &\le& (\|\k_i\|_2 + \|\G^{-1}\|_2 \|\s_i\|_2)^2 \\
  &\le& \biggl(1 + \frac{4}{\alpha\sqrt{r}-2}\biggr)^2 \\
  &=& \biggl(\frac{\alpha\sqrt{r}+2}{\alpha\sqrt{r}-2}\biggr)^2.
 \end{eqnarray*}
 The second inequality follows from 
 Lemma \ref{Lemm: inequalities about F, G and s} and 
 also $\|\k_i\|_2 \le 1$ due to Assumption \ref{Assump}(b). 
 Thus, the second constraints hold if 
 $
 \theta \le \bigl(\frac{\alpha \sqrt{r} - 2}{\alpha \sqrt{r} + 2} \bigr)^{2}.
 $
 Let 
 $
 \wb{\theta} = \bigl(\frac{\alpha \sqrt{r} - 2}{\alpha \sqrt{r} + 2} \bigr)^{2}.
 $
 It satisfies $0 < \wb{\theta} < 1$ because of $\alpha \sqrt{r} > 2$.
 Thus,  $\wb{\theta} \I$  is a feasible solution of $\transMVEE(\SC)$. 
 Accordingly, 
 for the optimal solution $\C^*$ and the feasible solution $\wb{\theta} \I$,
 we have
 \begin{equation*}
  \det(\C^*) \ge \det(\wb{\theta} \I)
   = \wb{\theta}^r
   = \Biggl(\frac{\alpha \sqrt{r} - 2}{\alpha \sqrt{r} + 2} \Biggr)^{2r}.
 \end{equation*}
\end{proof}

Lemma \ref{Lemm: upper bound for C} corresponds to Lemmas 2.5 and 2.6 of \cite{Gil15}.
The lemmas of that paper need to put a condition on the amount of noise 
in order to derive of the upper bound on $\det (\C^*)$, 
while this lemma does not need to do so.
This comes from the difference in the structures of near-separable matrices.
Our lemma handles a near-separable matrix of the form (\ref{Eq: P}).
It has a preferable structure wherein
the noise matrix contains an \by{r}{r} zero submatrix 
and this zero submatrix corresponds to a basis matrix.

\begin{lemm} \label{Lemm: upper bound for C}
 Suppose that $r \ge 2$. Then, 
 \begin{equation*}
  \det(\C^*) \le \Biggl(\frac{r - \lambda_j}{r-1}\Biggr)^{r-1} \lambda_j
 \end{equation*}
 for each $j=1, \ldots, r$.
\end{lemm}
\begin{proof}
 We derive an upper bound on the sum of the eigenvalues of $\C^*$
 that is equivalent to  $\mbox{tr}(\C^*)$.
 Since $\C^*$ is feasible for $\transMVEE(\SC)$,
 we have 
 \begin{equation*}
  \g_j^\top (\G^{-1})^\top \C^* \G^{-1} \g_j \le 1, 
   \ \mbox{equivalently,} \ 
   \|(\C^*)^{1/2} \e_j\|_2^2 \le 1 
   \ \mbox{for} \ j = 1, \ldots, r.
 \end{equation*}
 Thus, 
 \begin{eqnarray*}
  \lambda_1 + \cdots + \lambda_r 
   &=& \mbox{tr}(\C^*)  \\
   &=& \|(\C^*)^{1/2} \|_F^2 \\
   &=& \sum_{j=1}^r \|(\C^*)^{1/2}\e_j \|_2^2 \le r.
 \end{eqnarray*}
 The arithmetic-geometric mean inequality means that 
 $(a_1 \times \cdots \times a_r)^{1/r} \le (a_1 + \cdots + a_r) / r$ 
 holds for nonnegative real numbers $a_1, \ldots, a_r$.
 Therefore, we have, for each $j = 1, \ldots, r$,
 \begin{eqnarray*}
  \det(\C^*) = \lambda_1 \times \cdots \times \lambda_r  \le 
   \Biggl(\frac{r - \lambda_j}{r-1}\Biggr)^{r-1} \lambda_j.
 \end{eqnarray*}
\end{proof}

We denote
\begin{equation} \label{Eq: value of a}
 a = \Biggl(\frac{\alpha\sqrt{2} - 2}{\alpha\sqrt{2} + 2}\Biggr)^4.
\end{equation}
The value of $a$ is in $0 < a < 1$ when $\alpha > \sqrt{2}$.

\begin{prop} \label{Prop: Bound for eigenvalues}
 Let $\alpha$ be a real constant satisfying $\alpha > \sqrt{2}$ and 
 $r$ be any integer satisfying $r \ge 2$.
 Suppose that $\|\N\|_2 \le \frac{\sigma_{\mmin}({\F})}{\alpha \sqrt{r}}$ 
 and  $\K$ satisfies Assumption \ref{Assump}(b).
 Then,  $\lambda_j$ is bounded such that $1 - \sqrt{1-a} \le \lambda_j \le  1 + \sqrt{1-a}$
 for each $j = 1, \ldots, r$.
\end{prop}
\begin{proof}
 Lemmas \ref{Lemm: lower bound for C} and \ref{Lemm: upper bound for C} tells us that 
 $r$ and $\lambda_j$ need to satisfy the inequality
 \begin{equation} \label{Eq: inequality about r and lambda_j}
  \Biggl(\frac{\alpha \sqrt{r} - 2}{\alpha \sqrt{r} + 2} \Biggr)^{2r}
  \le \Biggl(\frac{r - \lambda_j}{r-1}\Biggr)^{r-1} \lambda_j.
 \end{equation}
 When $r=2$, it becomes $a \le (2-\lambda_j)\lambda_j$.
 Thus, it is necessary for $\lambda_j$ to satisfy $1 - \sqrt{1-a} \le \lambda_j \le  1 + \sqrt{1-a}$.
\end{proof}

\begin{remark} \label{Remark: difference in proof}
 Proposition \ref{Prop: Bound for eigenvalues} corresponds to Lemmas 2.6 and 2.7 of \cite{Gil15}.
 From the lower and upper bounds on $\det(\C^*)$,
 the lemmas of that paper construct an inequality that $r$ and $\lambda_j$ need to satisfy
 and determine the condition on $\lambda_j$ such that the inequality holds for all $r \ge 2$.
 In contrast, this proposition only considers the case of $r=2$ 
 for inequality (\ref{Eq: inequality about r and lambda_j}), 
 and determines the condition on $\lambda_j$.
\end{remark}

Since we have $\sigma_j = \lambda_j^{1/2}$ for the $j$th singular value $\sigma_j$ of $\G^\circ$,
this proposition gives the bounds on the singular values and condition number of $\G^\circ$.

\begin{coro} \label{Coro: bounds for singular values for transG}
 Suppose that the same conditions in Proposition \ref{Prop: Bound for eigenvalues} hold.
 Then, we have 
 \begin{eqnarray*}
  \sigma_{\mmin}(\G^\circ) 
   &\ge& \biggl(1-\sqrt{1-a} \biggr)^{1/2}, \\
  \sigma_{\mmax}(\G^\circ) 
   &\le&  \biggl(1+\sqrt{1-a} \biggr)^{1/2}, \\
  \kappa(\G^\circ) 
   &\le&  \Biggl(\frac{1+\sqrt{1-a}}{1-\sqrt{1-a}}\Biggr)^{1/2}. \\
 \end{eqnarray*}
\end{coro}

\subsubsection{Application of Theorem \ref{Theo: robustness of SPA} to $\P^\circ$}
\label{Subsec: applicability}

As we saw in Section \ref{Subsec: preliminaries}, 
$\P^\circ$ in step 3 of Algorithm \ref{Alg: precond SPA} is an \by{r}{m} near-separable matrix.
In the proposition below, we show that Theorem \ref{Theo: robustness of SPA} can apply to $\P^\circ$.
Here, we should note that
$\s^\circ_i$, $\p^\circ_{\IC(j)}$ and $\g^\circ_j$ in the proposition 
are the column vectors of $\S^\circ$, $\P^\circ$ and $\G^\circ$.

\begin{prop} \label{Prop: robustness in transformed space}
 Let $\wt{\A} = \A + \N$ for $\A \in \Real^{d \times m}$ and $\N \in \Real^{d \times m}$.
 Suppose that $r \ge 2$ and $\A$ satisfies Assumption \ref{Assump}.
 Let $\epsilon$ be such that $\|\s_i^\circ\|_2  \le \epsilon $ for all $i = 1, \ldots, m-r$.
 If 
 \begin{equation*}
  \|\N\|_2 \le \frac{\sigma_{\mmin}({\F})}{\alpha \sqrt{r}}
 \end{equation*}
 and $\alpha = 1225$,
 then, Algorithm \ref{Alg: precond SPA} with the input $(\wt{\A}, r)$ returns
 the output $\IC$ such that 
 there is an order of the elements in $\IC$ satisfying
 \begin{equation*}
   \|\p^\circ_{\IC(j)} - \g^\circ_j\|_2  \le (80 \kappa(\G^\circ)^2 + 1) \epsilon
 \end{equation*}
 for all $j = 1, \ldots, r$.
\end{prop}

\begin{proof}
 First, we show that a separable matrix $\G^\circ(\I,\K)\Pib$
 in $\P^\circ$ satisfies Assumption~\ref{Assump}.
 From Corollary \ref{Coro: bounds for singular values for transG},
 we have $\sigma_{\mmin}(\G^\circ) \ge (1 - \sqrt{1-a})^{1/2} > 1$.
 The last inequality strictly holds,
 since the value of $a$, which is of (\ref{Eq: value of a}), is in $0 < a < 1$ 
 due to $\alpha=1225$.
 Thus, we see that (a) of Assumption \ref{Assump} is satisfied.
 Furthermore, (b) of Assumption \ref{Assump} is satisfied
 since the assumption is put on the $\K$ of $\wt{\A}$.

 Next, we show that the size of the noise matrix $(\0, \S^\circ)\Pib$ in $\P^\circ$ 
 is within the range allowed by Theorem \ref{Theo: robustness of SPA}.
 Namely, we show that the inequality 
 \begin{equation} \label{Eq: inequality about allowed noise range}
  \|\s_i^\circ\|_2 < 
   \min \Biggl( \frac{1}{2\sqrt{r-1}}, \frac{1}{4} \Biggr)
   \frac{\sigma_{\mmin}(\G^\circ)}{1 + 80 \kappa(\G^\circ)^2}
 \end{equation}
 holds for $i = 1, \ldots, m-r$.
 We derive the upper bound on the left-side value 
 and the lower bound on the right-side value.
 The left-side value is bounded such that
 \begin{eqnarray*}
  \|\s_i^\circ\|_2 
   &=&  \|(\L^*)^{1/2}\s_i\|_2   \\
   &=&  \|(\L^*)^{1/2} \G \G^{-1} \s_i\|_2   \\
   &\le& \|(\L^*)^{1/2} \G\|_2 \|\G^{-1}\|_2 \|\s_i\|_2   \\
   &\le& \frac{4 (1 + \sqrt{1-a})^{1/2}}{\alpha \sqrt{r} -2}.
 \end{eqnarray*}
 The last inequality follows from Corollary \ref{Coro: bounds for singular values for transG} 
 and Lemma \ref{Lemm: inequalities about F, G and s}.
 The right-side value is bounded such that
 \begin{eqnarray*}
  \min \Biggl( \frac{1}{2\sqrt{r-1}}, \frac{1}{4} \Biggr)
   \frac{\sigma_{\mmin}(\G^\circ)}{1 + 80 \kappa(\G^\circ)^2} 
   &\ge&
  \min \Biggl( \frac{1}{2\sqrt{r-1}}, \frac{1}{4} \Biggr)
  \frac{(1-\sqrt{1-a})^{3/2}}{81+79\sqrt{1-a}} \\
   &>&
    \frac{1}{81}  \min \Biggl( \frac{1}{2\sqrt{r-1}}, \frac{1}{4} \Biggr)
  \frac{(1-\sqrt{1-a})^{3/2}}{1+\sqrt{1-a}}.
 \end{eqnarray*}
 The first inequality follows from Corollary \ref{Coro: bounds for singular values for transG}.
 Therefore, the inequality 
\begin{eqnarray} \label{Eq: relax inequality about allowed noise range}
 & & \frac{4 (1 + \sqrt{1-a} )^{1/2}}{\alpha \sqrt{r} -2}
  <
  \frac{1}{81} \min \Biggl( \frac{1}{2\sqrt{r-1}}, \frac{1}{4} \Biggr)
 \frac{(1-\sqrt{1-a})^{3/2}}{1+\sqrt{1-a}}  \nonumber \\
 \equivSym 
  & & 324b
  < 
  \min \Biggl( \frac{1}{2\sqrt{r-1}}, \frac{1}{4} \Biggr)(\alpha \sqrt{r}-2)
\end{eqnarray} 
 implies that of (\ref{Eq: inequality about allowed noise range}).
 Here, we denote 
 \begin{equation*}
  b = \Biggl( \frac{1 + \sqrt{1-a}}{1 - \sqrt{1-a}} \Biggr)^{3/2}.
 \end{equation*}
 Note that the value of $b$ is determined by $\alpha$, 
 since $a$ is given as (\ref{Eq: value of a}).

 We show that inequality 
 (\ref{Eq: relax inequality about allowed noise range}) 
 holds for any $r \ge 2$ when $\alpha = 1225$.
 In the case of $2 \le r \le 5$, it is sufficient to show 
 \begin{equation} \label{Eq: r <= 5}
  324b < \frac{\sqrt{2}}{4} \alpha- \frac{1}{2}.
 \end{equation}
 In the case of $r \ge 6$, the inequality becomes 
 \begin{equation} \label{Eq: r >= 6}
  324b < \frac{\alpha \sqrt{r}-2}{2 \sqrt{r-1}}.
 \end{equation}
 Let $f(x) = \frac{\alpha \sqrt{x}-2}{2 \sqrt{x-1}}$ for $x \ge 2$.
 The function $f$ attains its minimum at $x = \alpha^2 / 4$. 
 Thus, we can put a lower bound on the right-side value.
 \begin{eqnarray*}
  \frac{\alpha \sqrt{r}-2}{2 \sqrt{r-1}} 
  &\ge& f(\alpha^2 / 4) \\
  &=& \frac{1}{2}\sqrt{\alpha^2-4} \\
  &>& \frac{1}{2}\alpha - \frac{1}{2}.
 \end{eqnarray*}
 Accordingly, we see that inequality (\ref{Eq: r <= 5}) implies 
 that of (\ref{Eq: r >= 6}).
 For inequality (\ref{Eq: r <= 5}) with $\alpha = 1225$, 
 we have the relation 
 \begin{equation*}
  324b < 432.4 < 432.6
   < \frac{\sqrt{2}}{4} \alpha- \frac{1}{2}.
 \end{equation*}
 Therefore, inequality (\ref{Eq: relax inequality about allowed noise range})
 holds for any $r \ge 2$ when $\alpha = 1225$.
 This leads us to conclude that
 inequality (\ref{Eq: inequality about allowed noise range}) 
 holds for $i = 1, \ldots, m-r$.

\end{proof} 

Let us remark on the choice of  $\alpha = 1225$ in the proposition.
\begin{remark}
 The value of $b$ in (\ref{Eq: r <= 5})
 is given by the composite function $f$ in $x$ such that 
  $f = f_2 \circ f_1$ with 
 \begin{equation*}
  f_1(x) = \Biggl(\frac{\sqrt{2} x - 2}{\sqrt{2} x + 2}\Biggr)^4 \ \mbox{and} \
  f_2(x) = \Biggl( \frac{1 + \sqrt{1-x}}{1 - \sqrt{1-x}} \Biggr)^{3/2}.
 \end{equation*}
 The function $f$ is monotonically decreasing for $x \ge 2$, 
 and the function value approaches to $1$ as $x$ goes to infinity.
 Thus, $\alpha = 1225$ is the minimum integer 
 that satisfies the inequality (\ref{Eq: r <= 5}).
\end{remark}

In Proposition \ref{Prop: robustness in transformed space}, 
we evaluated the size of the basis error due to Algorithm \ref{Alg: precond SPA} 
in terms of $\P^\circ$.
The proof of Theorem~\ref{Theo: main} can be obtained by rewriting it 
as $\wt{\A}$ instead of $\P^\circ$.

\begin{proof}[(Proof of Theorem~\ref{Theo: main})] 
 The theorem supposes that $r \ge 2$ and $\A$ satisfies Assumption \ref{Assump}.
 Therefore, Proposition \ref{Prop: robustness in transformed space} tells us that,
 if $\|\N\|_2  \le \sigma_{\mmin}(\F) / \alpha \sqrt{r} $ and $\alpha = 1225$,
 then, Algorithm \ref{Alg: precond SPA} with the input $(\wt{\A}, r)$ returns
 the output $\IC$ such that there is an order of the elements in $\IC$ satisfying
 \begin{equation} \label{Eq: error of p and g in transformed space}
  \|\p^\circ_{\IC(j)} - \g^\circ_j\|_2  \le ( 80 \kappa(\G^\circ)^2 + 1 ) \epsilon
 \end{equation}
 for all $j = 1, \ldots, r$ where $\epsilon$ satisfies
 \begin{equation} \label{Eq: bound for noise in transformed space}
  \|\s_i^\circ\|_2 \le  \epsilon.
 \end{equation}
 The norm of $\s_i^\circ$ is bounded from above by using $\|\N\|_2$ such that 
 \begin{eqnarray*}
  \|\s_i^\circ\|_2 
   &=&
   \|(\L^*)^{1/2} \s_i\|_2  \\
   &\le&
   \|(\L^*)^{1/2}\G\|_2 \|\G^{-1}\|_2  \|\s_i\|_2  \\
   &\le&
    \frac{ 4 (1 + \sqrt{1-a})^{1/2}}{\sigma_{\mmin}(\G)}\|\N\|_2.
 \end{eqnarray*}
 The last inequality follows from Corollary \ref{Coro: bounds for singular values for transG} and 
 Lemma \ref{Lemm: bound for s}.
 We choose $\epsilon$ as 
 \begin{equation} \label{Eq: prime epsilon}
  \epsilon = \frac{ 4 (1 + \sqrt{1-a})^{1/2} }{\sigma_{\mmin}(\G)}\|\N\|_2.
 \end{equation}
 By this choice, the inequality (\ref{Eq: bound for noise in transformed space}) 
 is satisfied for all $i = 1, \ldots, m-r$.
 In what follows,
 we shall use $k$ to denote $\IC(j)$
 in (\ref{Eq: error of p and g in transformed space}) for simplicity.
 For the left-side of (\ref{Eq: error of p and g in transformed space}), we have 
 \begin{eqnarray*}
  \|\p^\circ_k - \g^\circ_j\|_2  
  &=& 
  \|(\L^*)^{1/2} (\p_k - \g_j)\|_2 \\
  &\ge&
   \sigma_{\mmin}((\L^*)^{1/2}) \|\p_k - \g_j\|_2 \\
  &\ge&
  \sigma_{\mmin}((\L^*)^{1/2}\G) \sigma_{\mmin}(\G^{-1}) \|\p_k - \g_j\|_2 \\
  &\ge&
   \frac{(1-\sqrt{1-a})^{1/2}}{\sigma_{\mmax}(\G)} \|\p_k - \g_j\|_2.
 \end{eqnarray*}
 The last inequality follows from  Corollary \ref{Coro: bounds for singular values for transG}.
 For the right side of (\ref{Eq: error of p and g in transformed space}), we have, 
 from Corollary \ref{Coro: bounds for singular values for transG} and 
 the choice of $\epsilon$ such as (\ref{Eq: prime epsilon}),
 \begin{eqnarray*}
  (1 + 80 \kappa(\G^\circ)^2) \epsilon
   &\le& 
   \Biggl(1 + \frac{80(1 + \sqrt{1-a})}{1 - \sqrt{1-a}} \Biggr)
   \frac{ 4 (1 + \sqrt{1-a})^{1/2} }{\sigma_{\mmin}(\G)} \|\N\|_2.
 \end{eqnarray*}
 Accordingly, 
 \begin{eqnarray}
  & & \frac{(1-\sqrt{1-a})^{1/2}}{\sigma_{\mmax}(\G)} \|\p_k - \g_j\|_2
   \le
   \Biggl(1 + \frac{80(1 + \sqrt{1-a})}{1 - \sqrt{1-a}} \Biggr)
   \frac{ 4 (1 + \sqrt{1-a})^{1/2} }{\sigma_{\mmin}(\G)} \|\N\|_2 \nonumber \\
  &\equivSym& 
   \|\p_k - \g_j\|_2 
   \le 
   \frac{4(81 + 79 \sqrt{1-a})(1 + \sqrt{1-a})^{1/2}}{(1 - \sqrt{1-a})^{3/2}} 
   \kappa(\G)\|\N\|_2. \label{Eq: error of p and g}
 \end{eqnarray}
 By Lemma \ref{Lemm: inequalities about F, G and s},
 the condition number of $\G$ is bounded from above by using that of $\F$ such that
 \begin{eqnarray*}
  \kappa(\G) = \frac{\sigma_{\mmax}(\G)}{\sigma_{\mmin}(\G)} 
   &\le& 
   \frac{\alpha \sqrt{r}}{\alpha\sqrt{r} - 2} \kappa(\F) + \frac{2}{\alpha\sqrt{r} - 2} \\
  &=&
   \frac{\alpha \sqrt{r}}{\alpha\sqrt{r} - 2} (\kappa(\F) + 1) - 1.
 \end{eqnarray*}
 Here, we consider the function $f(x) = \frac{\alpha \sqrt{x}}{\alpha\sqrt{x} - 2}$ 
 for $x \ge 2$.
 Since the function is monotonically decreasing for $x \ge 2$,
 we have 
 \begin{eqnarray*}
  \kappa(\G)
   \le 
   \frac{\alpha \sqrt{2}}{\alpha\sqrt{2} - 2} \kappa(\F) + 
   \frac{2}{\alpha\sqrt{2} - 2}.
 \end{eqnarray*}
 By using this inequality, we replace $\kappa(\G)$ in (\ref{Eq: error of p and g})
 with $\kappa(\F)$, and then 
 \begin{equation*}
   \|\p_k - \g_j\|_2 
    \le 
   \frac{4(81 + 79 \sqrt{1-a})(1 + \sqrt{1-a})^{1/2} }{(1 - \sqrt{1-a})^{3/2}} 
   \Biggl(\frac{\alpha \sqrt{2}}{\alpha\sqrt{2} - 2} \kappa(\F) 
   +  \frac{2}{\alpha\sqrt{2} - 2}\Biggl)\|\N\|_2.
 \end{equation*}
 Since $\alpha = 1225$ and $a$ is of the form (\ref{Eq: value of a}),
 the above inequality implies that 
 $\|\p_k - \g_j \|_2 \le (432\kappa(\F) + 1) \|\N\|_2$.
 By Lemma \ref{Lemm: relation of basis errors in original and transformed space},
 this inequality leads to 
 $\|\wt{\a}_k - \f_j \|_2 \le (432\kappa(\F) + 4) \|\N\|_2$.
\end{proof}

\section*{Acknowledgments}
The author would like thank Nicolas Gillis of 
Universit\'{e} de Mons who provided comments and suggestions 
on the draft version of this paper.
We also thank the anonymous referee for insightful comments
that enhanced the quality of this paper significantly.
This research was supported by the Japan Society for 
the Promotion of Science (JSPS KAKENHI Grant Number 15K20986, 26242027).

\bibliographystyle{plain}
\bibliography{main}

\end{document}